\documentclass[11pt]{article}

\usepackage[preprint]{acl}
\usepackage[normalem]{ulem}
\usepackage{times}
\usepackage{latexsym}
\usepackage[T1]{fontenc}
\usepackage[utf8]{inputenc}
\usepackage{microtype}
\usepackage{inconsolata}
\usepackage{graphicx}
\usepackage{amsmath}
\usepackage{amssymb}
\usepackage{amsthm}
\usepackage{booktabs}
\usepackage{multirow}
\usepackage{algorithm}
\usepackage{algorithmic}
\usepackage{enumitem}
\usepackage{authblk}
\usepackage{tabularx}
\usepackage{booktabs}
\usepackage{array}
\usepackage{fvextra}
\newtheorem{theorem}{Theorem}

\usepackage[most]{tcolorbox}

\newtcolorbox{takeawaybox}{
  colback=gray!5,
  colframe=gray!60,
  boxrule=0.5pt,
  arc=2pt,
  left=4pt,
  right=4pt,
  top=3pt,
  bottom=3pt,
  width=\linewidth
}

\newtcolorbox{promptbox}{
  enhanced,
  breakable,
  colback=gray!4,
  colframe=gray!55,
  boxrule=0.4pt,
  arc=2pt,
  left=5pt,
  right=5pt,
  top=4pt,
  bottom=4pt,
  before skip=6pt,
  after skip=8pt,
  fontupper=\small
}

\title{Implicit Reasoning for Large Language Model-based \\ Generative Recommendation}

\author{
    \textbf{Yinhan He}\textsuperscript{†}, 
    \textbf{Liam Collins}\textsuperscript{§}, 
    \textbf{Bhuvesh Kumar}\textsuperscript{§}, \\ 
    \textbf{Jundong Li}\textsuperscript{†}, 
    \textbf{Neil Shah}\textsuperscript{§}, 
    \textbf{Donald Loveland}\textsuperscript{§} 
}
\affil{\textsuperscript{†}University of Virginia, Charlottesville, VA, USA \vspace{-4mm}}
\affil{\textsuperscript{§}Snap Inc., Santa Monica, CA, USA \vspace{-4mm}}
\affil{\texttt{\{nee7ne, jl6qk\}@virginia.edu}, \texttt{\{lcollins2, bhuvesh, nshash, dloveland\}@snap.com}}

\begin{document}
\maketitle

\begin{abstract}
Large Language Models (LLMs) are increasingly adopted as backbones for Generative Recommendation (GR), promising access to pretrained world knowledge. Yet reliably invoking this knowledge for GR remains poorly understood. 
A key obstacle is that LLM-based GR typically represents items with Semantic IDs (SIDs), disrupting LLMs’ natural-language reasoning interface because these tokens are unseen by the LLM during pretraining. Existing approaches address this with expensive multi-stage pipelines that ground SIDs and elicit explicit rationales, but offer limited insight into when and why each stage is necessary. In this work, we 
systematically decompose explicit reasoning training pipelines for LLM-based GR, revealing three key limitations: weakened world-knowledge verbalization, misalignment between SID and natural-language token embedding spaces, and sensitivity to rationale quality---all of which hurt explicit reasoning performance.
To circumvent these issues, we propose \textsc{PauseRec}, a lightweight implicit reasoning paradigm tailored for GR. 
\textsc{PauseRec} is exceptionally practical, avoiding costly
reasoning trace acquisition and 
reasoning alignment training, leading to a multitude of benefits: (1) it outperforms standard explicit CoT 
methods by up to 6.22\%, (2) it reduces training cost by up to \textcolor{red}{65\%} GPU hours, and (3) it speeds up inference by up to \textcolor{red}{71.3\%}. These results position \textsc{PauseRec} as a lightweight alternative to explicit rationale generation, enabling more effective and efficient LLM-based GR\footnote{Work done when Yinhan He was a Research Intern at Snap Inc..}. 
\end{abstract}

\section{Introduction}
\label{sec:intro}

Large Language Models (LLMs) have recently been adopted as backbones for Generative Recommendation (GR), enabling LLM-based GR systems that formulate recommendation as conditional generation: an LLM reads a user history and generates the next item \citep{hua2023up5, bao2023tallrec, rajput2024recommender}. The appeal of LLMs for GR lies in their pretrained world knowledge~\cite{zhao2023survey, huang2023towards,yu2024kola}.
In principle, this knowledge can help infer semantic relationships among historical items, identify a user's latent intent, and map that intent to plausible next items beyond memorized co-occurrences~\citep{wang-etal-2025-agrec, zhang-etal-2025-cove}. 
Yet the process of efficiently and effectively accessing LLMs’ pretrained knowledge 
for GR remains poorly understood~\cite{zhang2026thinking}.

A key obstacle to leveraging an LLM’s world knowledge for GR centers on item representation. Specifically, LLM-based GR systems typically represent items with Semantic IDs (SIDs), i.e., short sequences of special tokens derived from items’ semantic relations~\citep{rajput2023recommender}. SIDs make item generation tractable given their compactness, but they are not natural-language expressions and reside outside the pretrained LLM vocabulary~\citep{li-etal-2021-personalized}.
This creates a mismatch: LLMs access world knowledge through natural language, while the recommendation task is to generate a non-linguistic SID conditioned on other non-linguistic SIDs. We therefore ask: 
\textit{How can pretrained LLM world knowledge be effectively leveraged to improve recommendation over SID tokens?}

Following the broader LLM literature~\cite{yu2024kola, petroni2019language}, one natural answer to this question is leveraging explicit Chain-of-Thought (CoT) reasoning~\citep{wei2022chain, kojima2022large}\footnote{In this work, we use both ``reasoning'' and ``rationales'' to refer to LLMs’ Chain-of-Thought (CoT) process, i.e., the intermediate, step-by-step traces that LLMs generate before producing a final answer~\cite{wei2022chain}.}. Explicit CoT has been shown to improve LLM performance on a range of knowledge-intensive domains, including mathematics~\cite{imani2023mathprompter}, science~\cite{truhn2023large}, and coding~\cite{jiang2026survey}. 
For LLM-based GR, previous methods have pursued a similar goal through multi-step training pipelines. These pipelines typically ground LLMs in SIDs via continual pretraining (CPT) on natural-language item descriptions, optimize next-item prediction with supervised finetuning (SFT), elicit explicit rationales through SFT over reasoning trajectories (which we refer to as CoT SFT) 
, and refine model responses with reinforcement learning (RL) post-training~\citep{liu2025onerec, yu2025thinkrec, liang2026generative}.
Yet existing work provides limited insight into when these stages are necessary and why they help. 
Given each stage's high computation cost, understanding these questions is critical to both justifying the full workflow and identifying more efficient alternatives.

\begin{figure*}
    \centering
    \includegraphics[width=1\linewidth]{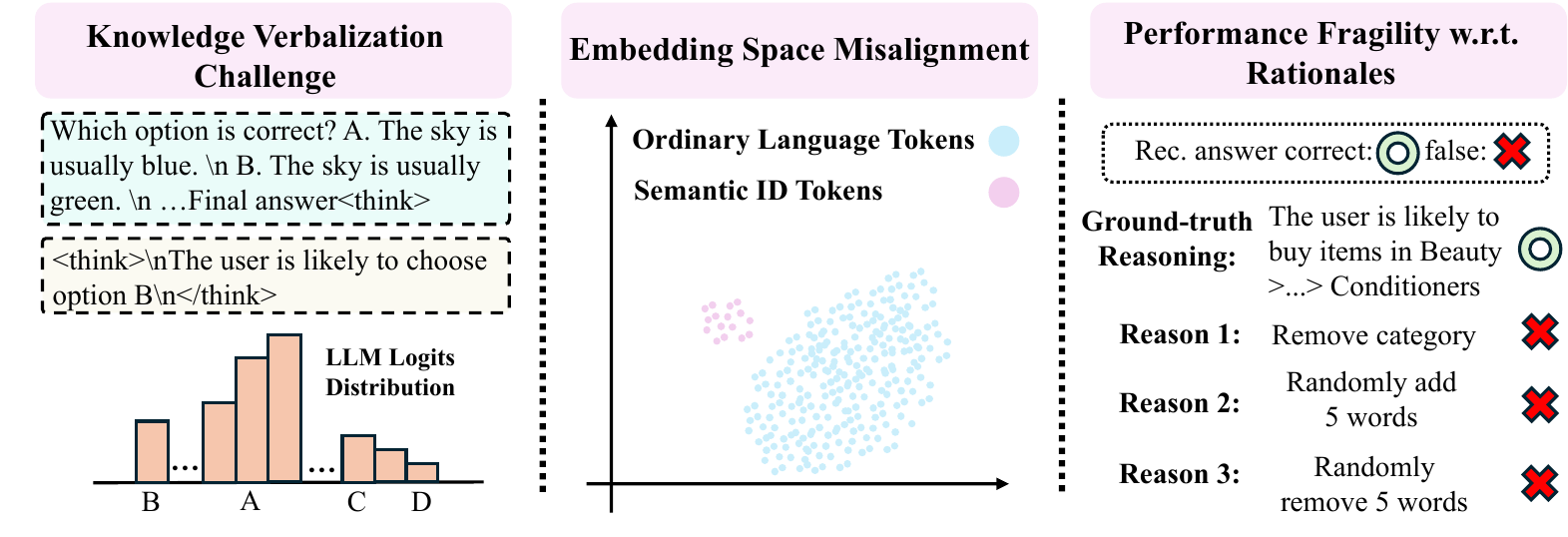}
    \caption{The three identified limitations for explicit CoT in SID-based GR. CoT SFT weakens world-knowledge verbalization (left), separates natural-language and SID embedding spaces (middle), and makes recommendation quality sensitive to rationale perturbations (right), motivating an implicit alternative to verbal rationales. 
    }
    \label{fig:placeholder}
    \vspace{-0.1in}
\end{figure*}
To address this gap, we first analyze explicit reasoning pipelines for LLM-based GR, examining each stage’s contribution and necessity. We begin with CPT 
stage, finding that CPT-trained models can recover coarse item categories but often struggle to identify titles or fine-grained categories, indicating that grounding provides a real but incomplete semantic signal. We then test whether CoT SFT with various reasoning formats, including template-based reasoning and teacher-generated reasoning, can improve recommendation performance. Across these variants, 
CoT SFT consistently \textit{\underline{underperforms}} simple next-item SFT. Performance gains from explicit CoT 
emerge only after expensive RL post-training.

To explain this discrepancy, we identify three limitations of explicit reasoning. First, we find that CoT SFT makes pretrained world knowledge harder to verbalize under standard decoding, even though this knowledge remains recoverable from the model’s logits. Second, we show that text and SID token embeddings become geometrically separated during training. Our theoretical analysis proves that this separation limits the extent to which reasoning expressed in natural-language tokens can shape the final SID prediction. 
Third, we demonstrate that recommendation performance is sensitive to superficial perturbations of the ground-truth rationales. Together, these findings suggest that explicit rationales are a brittle interface for exploiting LLM knowledge in LLM-based GR.

To circumvent the aforementioned challenges, we propose \textsc{PauseRec}, a lightweight \textit{implicit} reasoning framework for SID-based GR. Instead of crafting ground-truth natural-language rationales with expensive teacher models and training the model to generate those rationales, \textsc{PauseRec} inserts a short sequence of trainable \texttt{<pause>} tokens before SID generation. The \texttt{<pause>} token is initialized and pretrained to connect language and SID representations, then optimized only through the final next-item prediction objective, 
giving the model latent computation steps that directly shape SID prediction. \textsc{PauseRec} addresses the three issues of explicit reasoning pipelines by (i) 
removing reliance on verbalizing pretrained knowledge, (ii) bridging the text-SID representation gap through a trainable \texttt{<pause>} token,
 and (iii) avoiding brittle rationale supervision. On multiple Amazon review datasets, \textsc{PauseRec} outperforms SFT and CoT-based methods by up to 6.22\%, while substantially simplifying explicit reasoning pipelines; it reduces training cost by up to \textcolor{red}{65\%} GPU hours and speeds up inference by \textcolor{red}{71.3\%}, positioning implicit reasoning as a stronger and more efficient alternative for LLM-based GR.
Our contributions are as follows:
\begin{itemize}[leftmargin=*, itemsep=0pt, topsep=2pt]
    \item \textbf{Diagnostic analysis.} We decompose explicit reasoning pipelines for LLM-based GR and identify why they fail without RL post-training, including incomplete SID grounding, weakened world-knowledge verbalization, text--SID embeddings mismatch, and sensitivity to rationale formats.
    \item \textbf{Implicit reasoning framework.} We introduce a novel pipeline  termed \textsc{PauseRec}, which uses trainable \texttt{<pause>} tokens to elicit latent reasoning without rationale supervision.
    \item \textbf{Empirical evaluation.} Across three Amazon review datasets, \textsc{PauseRec} improves over standard SFT and CoT-based methods by up to 6.22\% while reducing training and inference overhead.
\end{itemize}

\section{Preliminaries}
\label{sec:prelim}

\subsection{Problem Formulation}

Following the GR literature~\cite{liu2025onerec}, we consider the sequential recommendation task. Let $\mathcal{I}$ denote the set of all items. Given a user's $n$ chronologically ordered interaction history $H = [i_1, i_2, \ldots, i_n]$ where $i_j \in \mathcal{I}$, the task is to predict the next item $i_{n+1}$ that the user will interact with.
Following recent work \citep{rajput2024recommender, bao2023tallrec}, LLM-based GR represents each item $i \in \mathcal{I}$ with a Semantic ID (SID), i.e., a sequence of tokens $s_i = [s_i^{(1)}, s_i^{(2)}, \ldots, s_i^{(L)}]$ of length $L$ that are added to the LLM's vocabulary.
 Recommendation can then be framed as conditional generation:
\begin{equation}
p(i_{n+1} | H) = p(s_{i_{n+1}} | \text{Prompt}(H))
\end{equation}
where $\text{Prompt}(H)$ converts the interaction history into a natural-language prompt listing past items (and optionally metadata). All methods in this paper share this generative formulation; they differ in how reasoning is inserted before SID prediction.

\subsection{Existing Explicit CoT Pipelines for GR}\label{sec: traditional_cot_rec}

We introduce the multiple training stages of existing explicit reasoning pipelines~\cite{liu2025onerec,liang2026generative} for GR as follows:

\noindent\textbf{Continual Pretraining (CPT).} The LLM is finetuned on an interleaved corpus of SIDs and item descriptions, with only SID token embeddings trainable. This stage grounds item semantics into SID token embeddings. Given item $i$ with description $d_i$, the model is trained on:
\begin{equation}
\mathcal{L}_{\text{CPT}} = -\mathbb{E}_{(s_i, d_i)} \left[\log p(s_i | d_i) + \log p(d_i | s_i)\right]
\end{equation}

\noindent\textbf{Next-item Supervised finetuning (SFT).} The CPT model is finetuned on user-items interaction histories to predict the next item by generating its SID:
\begin{equation}
\mathcal{L}_{\text{SFT}} = -\mathbb{E}_{(H, i_{n+1})} \left[\log p(s_{i_{n+1}} | \text{Prompt}(H))\right]
\end{equation}

\noindent\textbf{CoT SFT.} After SFT, the model is finetuned to generate natural-language rationales before the target SID. The training objective pairs each history $H$, rationale $r$, and next item $i_{n+1}$ as
\begin{equation}
\begin{aligned}
&\mathcal{L}_{\text{Reasoning}} =\\ &-\mathbb{E}_{(H, r, i_{n+1})} \left[\log p(r, s_{i_{n+1}} | \text{Prompt}(H))\right]
\end{aligned}
\end{equation}
Here, rationales are method-specific: some~\cite{liu2025onerec} use reasoning templates, while others~\cite{liang2026generative} utilize a teacher LLM.

\noindent\textbf{Reinforcement Learning (RL) Post-training.} Existing methods further apply RL to optimize recommendation rewards 
directly~\citep{liu2025onerec,yu2025thinkrec,liang2026generative}, though this stage is computationally expensive.

\section{Contributions of the Training Stages}

Given the current gap in understanding when and why different training stages make CoT effective for GR, we analyze the role of each stage in Section~\ref{sec: traditional_cot_rec}. We focus on CPT and CoT SFT here; next-item SFT and RL are evaluated in Section~\ref{sec:experiments}.
\subsection{CPT: Can LLMs Recover SID Semantics?}
\label{sec:understanding}

The primary aim of CPT is to ground SID semantics in LLMs, based on the premise that LLMs can reason over SIDs only after understanding their semantics. Before examining reasoning-related stages, we ask how much item-level semantic information an LLM recovers from SIDs after CPT.

\noindent\textbf{Experimental Design.} 
We train a Qwen3-1.7B~\cite{qwen3technicalreport} backbone on Amazon Beauty~\cite{ni2019justifying} with CPT for 2 epochs,  
where each SID is paired with its name and category during training. After CPT, we test whether the model can generate (1)~\textbf{item titles} and (2)~\textbf{item categories}\footnote{In Amazon Beauty~\cite{ni2019justifying}, categories are three-level paths, e.g., ``Beauty > Hair Care > Conditioners.''} at 1-, 2-, and 3-level granularity. We prompt with each test SID and measure exact-match accuracy; prompts and decoding are in Appendix~\ref{app:sid_decode_prompts}.

\noindent\textbf{Results and Analysis.}
\begin{table}[t]
    \centering
    \small
    \resizebox{\columnwidth}{!}{%
    \begin{tabular}{lccc}
    \toprule
    \textbf{Metric} & \textbf{Beauty} & \textbf{Sports} & \textbf{Toys} \\
    \midrule
    1-level Category & 98.00 & 97.83 & 99.56 \\
    2-level Category & 27.80 & 40.49 & 11.85 \\
    Category (Full) & 7.19 & 1.30 & 6.74 \\
    Title Recovery & 0.00 & 0.06 & 0.14 \\
    \bottomrule
    \end{tabular}
    }
    \caption{SID metadata recovery after CPT. The model recovers coarse one-level categories almost perfectly, but fails on item titles and fine-grained categories, showing that SID grounding provides partial semantic information rather than precise item-level understanding.}
    \label{tab:metadata}
\end{table}
From Table~\ref{tab:metadata}, we observe that (1) \textit{Fine-grained understanding is poor:} title recovery stays near 0\% and full-category accuracy remains below 7.2\% on all datasets, so item-level semantics are largely unrecovered. (2) \textit{Coarse category signal is strong:} 1-level category accuracy reaches highest 99.6\%
and 2-level accuracy up to 40.5\%, indicating that CPT captures broad categorical structure. These results show that LLMs associate SIDs with semantics from pretraining, but only at a coarse level. We next test whether CoT can convert this signal into better SID prediction.


\subsection{CoT SFT: The Failure of Explicit CoT}
\label{sec:cot_failure}

Here, we investigate if CoT SFT improves GR. 

\noindent\textbf{Experimental Design.}
We perform CoT SFT on Qwen3-1.7B~\cite{qwen3technicalreport} after CPT and SFT on Amazon Beauty~\cite{ni2019justifying}, using template-based, teacher-generated, rejection-sampled, and format-restricted rationales. For template-based rationales, we use: (1) \underline{Template-Category}: ``The user is likely to buy items in the \{target item category\} category.'' and (2) \underline{Template-Extended}: ``The user demonstrates interest in \{frequent categories\} products. By identifying the user's preference in \{characteristics\}, we can predict the user's purchase of \{target categories\}, for example, a \{item title\}.'' For teacher-generated rationales, we use Gemini 3.1 Flash-Lite and Pro~\cite{team2023gemini} to produce \underline{free-form} traces. For rejection sampling, Gemini 3.1 Flash-Lite generates multiple traces per sample, and we select either the trace with the highest target-SID logits (\underline{Gemini 3.1 Flash-Lite Rejection}) or the trace Gemini 3.1 Pro judges to best connect the user history to the target item (\underline{Gemini 3.1 FL Gemini Rejection}). Finally, \underline{format-restricted} traces impose reasoning constraints, e.g., rationales must reference SIDs. Sample rationales and prompts are in Appendix~\ref{app:implementation_details}.

\noindent\textbf{Results and Analysis.}
\begin{table}[t]
\centering
\small
\resizebox{\columnwidth}{!}{%
\begin{tabular}{lcc}
\toprule
\textbf{Method} & \textbf{Hit@5} & \textbf{NDCG@5} \\
\midrule
Next-item SFT (Baseline) & 0.0533 & 0.0381  \\
\midrule
Template-Category & 0.0492 & 0.0343 \\
Template-Extended &  0.0425 & 0.0282  \\
\midrule
Gemini 3.1 Flash-Lite Free-form & 0.0520 & 0.0361 \\
Gemini 3.1 Pro Free-form & 0.0422 & 0.0276 \\
Gemini 3.1 Flash-Lite Rejection & 0.0526 & 0.0362 \\
Gemini 3.1 FL Gemini Rejection & 0.0524 & 0.0367 \\
Gemini 3.1 Flash-Lite Restricted & 0.0449 & 0.0307  \\
\bottomrule
\end{tabular}}
\caption{CoT SFT variants on Amazon Beauty. Explicit rationales fail to outperform simple next-item SFT across rationale variants, indicating that rationale supervision alone does not reliably improve GR.}
\label{tab:cot_results}
\end{table}
Table~\ref{tab:cot_results} shows that explicit CoT variants underperform next-item SFT. The strongest variant, Gemini rejection sampling, remains below the baseline (0.0524 vs.\ 0.0533 Hit@5), while weaker teacher-generated variants lose over 20\% relative Hit@5, so CoT SFT alone does not reliably improve SID prediction. This pattern contrasts with CoT's success on language tasks: prior LLM-based GR work that reports gains from in-text reasoning relies on expensive RL with verifiable rewards (RLVR) after CoT SFT~\citep{liu2025onerec,yu2025thinkrec}. While RLVR can recover performance, it requires multiple rollout trajectories per step and is substantially more expensive than next-item SFT, which raises the question: \textit{why does CoT SFT fail for SID-based GR?}


\section{Diagnosis of CoT SFT Limitations}
\label{sec:diagnosis}

To understand why explicit CoT SFT fails, we conduct diagnostic studies and identify three limitations of the CoT SFT stage.

\subsection{Difficulty Verbalizing World Knowledge}

\noindent \textbf{Finding.} CoT SFT does not erase LLMs' world knowledge, but makes them difficult to verbalize.  

\noindent\textbf{Experimental Design.} We evaluate Qwen3-1.7B~\cite{qwen3technicalreport} after CoT SFT on representative language tasks benchmarks MMLU~\citep{hendrycks2020measuring}, HellaSwag~\citep{zellers2019hellaswag}, PIQA~\citep{bisk2020piqa}, and ARC-Challenge~\citep{clark2018think} in multiple-choice format. We report \textit{text-match} accuracy (exact A/B/C/D generation) and \textit{logit-based} accuracy (whether the correct choice has the highest logit).

\begin{table}[t]
\centering
\small
\resizebox{\columnwidth}{!}{%
\begin{tabular}{lccc}
\toprule
\textbf{Dataset} & \textbf{Base} & \textbf{Text Match} & \textbf{Logit-based} \\
\midrule
MMLU & 57.60 & 0.10 & 55.20 \\
HellaSwag & 61.00 & 9.00 & 60.40 \\
PIQA & 69.90 & 5.60 & 69.90 \\
ARC-C & 76.40 & 0.80 & 73.80 \\
\bottomrule
\end{tabular}
}
\caption{General-language reasoning accuracy after recommendation CoT SFT. Text-match accuracy collapses while logit-based accuracy remains close to the base model, suggesting that answer information is still present in logits but is no longer reliably verbalized.}
\label{tab:forgetting}
\end{table}

\noindent\textbf{Results and Analysis.} Table~\ref{tab:forgetting} shows text-match accuracy degrades after CoT SFT, while logit-based accuracy remains close to the base model on all benchmarks. 
It shows that the LLM's world knowledge remains primarily in logit space and is hard to verbalize in explicit natural language text format. We next examine whether this text--SID interface mismatch is also in the token embedding space.

\subsection{Text--SID Embedding Misalignment}
\noindent\textbf{Finding.} SID and natural-language tokens become geometrically separated in the token embedding space, causing difficulty for LLM to unify text and SIDs under a coherent rationale (see results and analysis for specific reasons).

\noindent\textbf{Experimental Design.} We visualize token embeddings after SID initialization, CPT, SFT, and CoT SFT using PCA~\cite{jolliffe2025principal}, comparing ordinary text tokens with SID tokens.

\begin{figure}[t]
  \centering
  \includegraphics[width=\columnwidth]{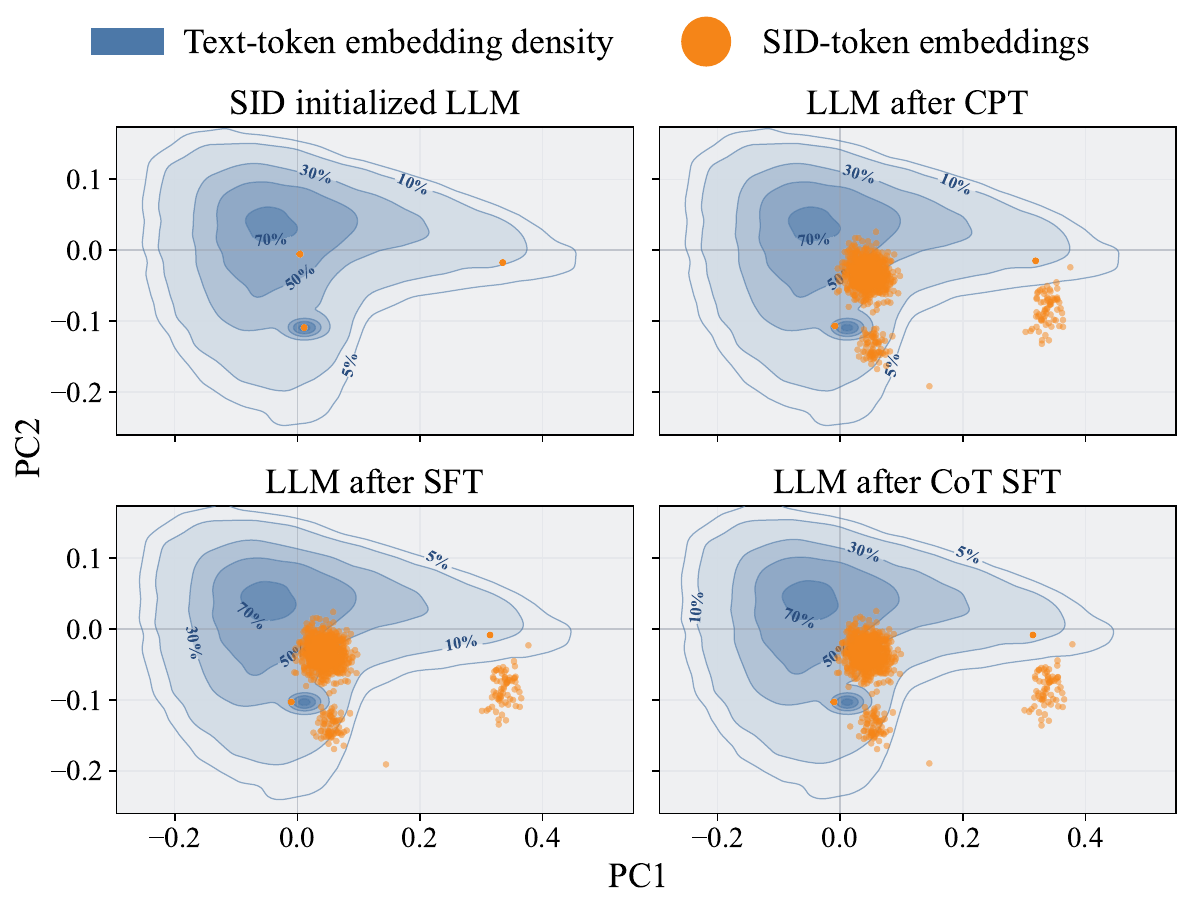}
  \caption{PCA of text and SID token embeddings across training stages. SID tokens drift away from ordinary text tokens as training progresses, indicating limited embedding space overlap between text and SIDs.}
  \label{fig:embedding_pca}
\end{figure}

\noindent\textbf{Results and Analysis.} Figure~\ref{fig:embedding_pca} shows text and SID embeddings diverging across stages, and the gap is already pronounced after CPT and continue to slightly expand during SFT and CoT SFT.  This token embedding discrepancy suggests difficulty in unifying language and SIDs in one coherent representation. Specifically, Appendix~\ref{app:text_sid_theorem} proves that when text- and SID-induced hidden-state directions are weakly coupled, updates driven by natural-language rationales can only weakly shift the logits over SID tokens, so explicit CoT has limited leverage on the final recommendation. 

\subsection{Performance Fragility w.r.t. Rationales}

\noindent\textbf{Finding.} After CoT SFT, recommendation performance is highly sensitive to the rationale text at inference, even when generated reasoning only slightly deviates from the ground-truth rationale.

\noindent\textbf{Experimental Design.} We test CoT SFT models with ground-truth rationales and controlled perturbations---removing the target item category, randomly dropping five words, or randomly adding five noise words---and measure Hit@5 and NDCG@5 under each setting.

\begin{table}[t]
\centering
\small
\resizebox{0.83\columnwidth}{!}{%
\begin{tabular}{lcc}
\toprule
\textbf{Reasoning Variant} & \textbf{Hit@5} & \textbf{NDCG@5} \\
\midrule
Ground-truth & 0.1165 & 0.0836 \\
Remove category & 0.0540 & 0.0376 \\
Drop 5 words & 0.0950 & 0.0834 \\
Add 5 words & 0.1145 & 0.0682 \\
\bottomrule
\end{tabular}}
\caption{Rationale perturbation sensitivity on Amazon Beauty~\cite{ni2019justifying}. Removing the target category more than halves performance, and even small word-level perturbations affect accuracy, showing that explicit CoT relies on brittle rationale cues. 
}
\label{tab:fragility}
\end{table}

\noindent\textbf{Results and Analysis.} Table~\ref{tab:fragility} shows that performance is highly sensitive to rationale content. Removing the target category more than halves Hit@5 (0.1165 to 0.0540) and NDCG@5 (0.0836 to 0.0376). Surface perturbations also matter: dropping five words reduces Hit@5 by 18.5\%, while adding five noise words reduces NDCG@5 by 18.4\%. Explicit CoT thus depends on brittle rationale cues, especially whether the text preserves semantics needed for the target SID.

\subsection{Summary of Findings}
Our diagnostics expose three CoT failures: \textbf{weakened verbalization} leaves answer signals in logits but weakens decoding; \textbf{text--SID embedding misalignment} limits rationale effects on SID logits; and \textbf{fragile rationales} make metrics sensitive to small edits. This motivates \textit{implicit reasoning in latent space}: learned \texttt{<pause>} tokens bridge language to SIDs without decoding brittle intermediate natural language reasoning text.

\section{Methodology: \textsc{PauseRec}}
\label{sec:method}
\begin{figure*}
    \centering
    \includegraphics[width=\linewidth]{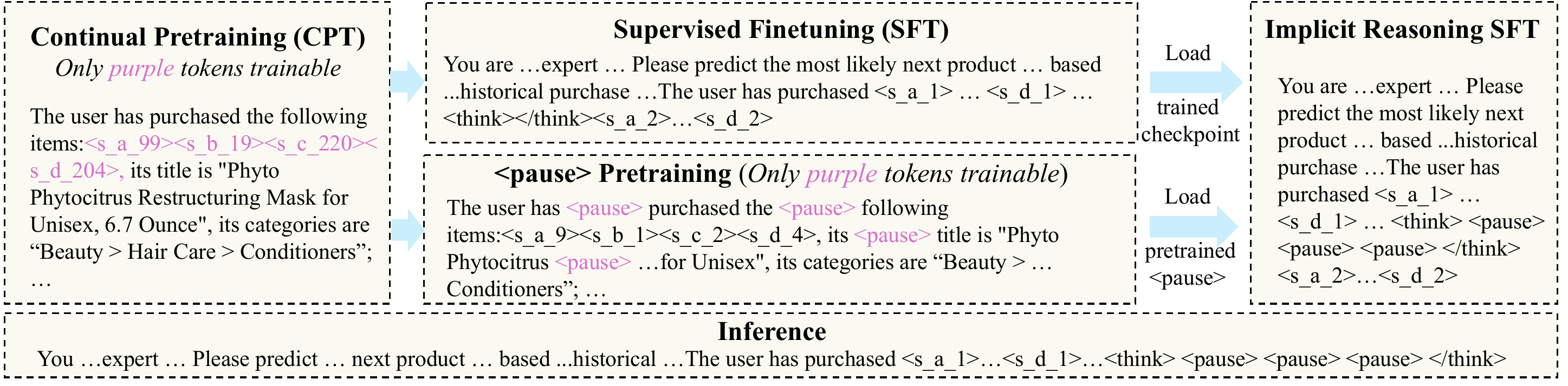}
    \caption{Overview of \textsc{PauseRec}. Instead of generating explicit rationales and applying RL post-training, \textsc{PauseRec} pretrains a \texttt{<pause>} token to bridge text and SID representations, then inserts pause tokens before SID generation and trains them only through the final next-item prediction loss.}
    \label{fig:pauserec_overview}
\end{figure*}
Motivated by Section~\ref{sec:diagnosis}, we propose \textsc{PauseRec}, an implicit reasoning method for LLM-based GR that keeps the CPT and next-item SFT stages of explicit pipelines but replaces CoT SFT and RL with pause-based latent computation. 

\subsection{Overview of \textsc{PauseRec}}

As illustrated in Fig.~\ref{fig:pauserec_overview}, we perform CPT (Section~\ref{sec: traditional_cot_rec}), then conduct next-item SFT and \texttt{<pause>} pretraining in parallel on the same CPT checkpoint. The SFT branch follows Section~\ref{sec: traditional_cot_rec}; the pause branch finetunes on the CPT corpus with \texttt{<pause>} tokens injected at random text positions so the token learns semantic transitions between language and SID tokens. We then load the pretrained \texttt{<pause>} embedding into the SFT checkpoint and run implicit reasoning SFT on next-item data with $k$ pauses inserted between user history and target SID, optimizing only SID positions.

\textsc{PauseRec} addresses the three CoT failures above: (1) \textbf{Computation without verbalization} via latent \texttt{<pause>} steps that need not be decoded as natural language; (2) \textbf{Bridging embedding spaces} via CPT-grounded pause pretraining (Appendix~\ref{app:pauserec_embedding_viz}  visualizes the trained \texttt{<pause>} token positioned between embedding spaces); and (3) \textbf{Avoiding rationale supervision} by masking loss on pause positions and optimizing only target SID.

\subsection{\texttt{<pause>} Token Initialization}

We add \texttt{<pause>} to the vocabulary and initialize its embedding at the mean of all token embeddings after CPT, with variance set to $10^{-9}$ times the embedding variance (equivalent to a near-deterministic start at the vocabulary center):
$
\mathbf{e}_{\texttt{<pause>}}^{(0)} = \frac{1}{|\mathcal{V}|} \sum_{v \in \mathcal{V}} \mathbf{e}_v
$, 
where $\mathcal{V}$ is the full vocabulary. This center initialization gives \texttt{<pause>} a neutral starting point between text and SIDs.

\subsection{Two-Stage Training}

\noindent\textbf{Stage 1: \texttt{<pause>} Token Pretraining.} Starting from the CPT checkpoint, we finetune on the CPT corpus with \texttt{<pause>} inserted at random positions covering 10\% of each sequence (Fig.~\ref{fig:pauserec_overview}). Only $\mathbf{e}_{\texttt{<pause>}}$ is trainable; all other parameters remain frozen. This concentrates updates on the bridge token while preserving grounded SID embeddings and the pretrained language backbone.

\noindent\textbf{Stage 2: Implicit Reasoning SFT.} We load the pretrained $\mathbf{e}_{\texttt{<pause>}}$ into the SFT checkpoint and append $k$ pauses between user history and target:
\begin{equation}
x' = \text{Prompt}(H) \| \underbrace{\texttt{<pause>}, \ldots, \texttt{<pause>}}_{k \text{ times}}
\end{equation}
The LLM is finetuned with loss masked at \texttt{<pause>} positions; only target SID tokens are optimized:
\begin{equation}
\vspace{-0.05in}
\mathcal{L}_{\text{implicit}} = -\sum_{l=1}^{L} \log p_\theta\left(s_{n+1}^{(l)} \mid x', s_{n+1}^{(1:l-1)}\right)
\vspace{-0.01in}
\end{equation}
By not imposing loss on pause positions, we avoid imitating a fixed teacher rationale distribution and instead let the model use pauses only when they improve SID prediction. In practice, pause slots act as task-specific latent scratch space between the textual history and discrete SID outputs.
See Appendix~\ref{app:pauserec_prompts} for sample prompts and formatted training text for each stage of \textsc{PauseRec}.

\subsection{Inference}
\begin{table*}[t]
    \centering
    \normalsize
    \setlength{\tabcolsep}{4pt}
    \resizebox{\textwidth}{!}{%
    \begin{tabular}{lcccccccccccc}
    \toprule
    & \multicolumn{4}{c}{\textbf{Beauty}} & \multicolumn{4}{c}{\textbf{Sports}} & \multicolumn{4}{c}{\textbf{Toys}} \\
    \cmidrule(lr){2-5} \cmidrule(lr){6-9} \cmidrule(lr){10-13}
    \textbf{Method} & H@5 & H@10 & N@5 & N@10 & H@5 & H@10 & N@5 & N@10 & H@5 & H@10 & N@5 & N@10 \\
    \midrule
    GRU4Rec & 0.0395 & 0.0584 & 0.0265 & 0.0326 & 0.0190 & 0.0312 & 0.0122 & 0.0161 & 0.0330 & 0.0490 & 0.0228 & 0.0279 \\
    SASRec & 0.0402 & 0.0607 & 0.0254 & 0.0320 & 0.0199 & 0.0301 & 0.0106 & 0.0141 & 0.0448 & 0.0626 & 0.0300 & 0.0358 \\
    BERT4Rec & 0.0232 & 0.0396 & 0.0146 & 0.0199 & 0.0102 & 0.0175 & 0.0065 & 0.0088 & 0.0215 & 0.0332 & 0.0131 & 0.0168 \\
    HGN & 0.0319 & 0.0536 & 0.0196 & 0.0266 & 0.0183 & 0.0313 & 0.0109 & 0.0150 & 0.0326 & 0.0517 & 0.0192 & 0.0254 \\
    HSTU & 0.0424 & 0.0652 & 0.0280 & 0.0353 & 0.0268 & 0.0343 & 0.0173 & 0.0226 & 0.0366 & 0.0566 & 0.0245 & 0.0309 \\
    TIGER & 0.0405 & 0.0623 & 0.0267 & 0.0337 & 0.0215 & 0.0347 & 0.0137 & 0.0179 & 0.0337 & 0.0547 & 0.0209 & 0.0276 \\
    \midrule
    ReaRec & 0.0450 & 0.0704 & 0.0262 & 0.0344 & 0.0214 & 0.0332 & 0.0116 & 0.0154 & 0.0523 & 0.0764 & 0.0298 & 0.0376 \\
    Next-item SFT & 0.0533 & 0.0733 & 0.0381 & 0.0445 & 0.0287 & 0.0409 & 0.0198 & 0.0237 & 0.0565 & 0.0781 & 0.0402 & 0.0471 \\
    OneRec-Think & \underline{0.0563} & \textbf{0.0791} & \underline{0.0398} & \textbf{0.0471} & \underline{0.0288} & \underline{0.0412} & \underline{0.0199} & \underline{0.0239} & \underline{0.0579} & \underline{0.0797} & \underline{0.0412} & \underline{0.0482} \\
    \midrule
    \textbf{\textsc{PauseRec}} & \textbf{0.0568} & \underline{0.0746} & \textbf{0.0401} & \underline{0.0467} & \textbf{0.0294} & \textbf{0.0422} & \textbf{0.0203} & \textbf{0.0245} & \textbf{0.0615} & \textbf{0.0838} & \textbf{0.0434} & \textbf{0.0509} \\
    \bottomrule
    \end{tabular}
    }
    \caption{Main recommendation results on three Amazon datasets. \textsc{PauseRec} consistently improves over next-item SFT and exceeds the baselines on most metrics; boldface and underlining mark the best and second-best results.}
    \label{tab:main_results}
    \end{table*}
At test time, we use the same prompt template as implicit-reasoning SFT (Appendix~\ref{app:pauserec_prompts}), insert $k$ literal \texttt{<pause>} tokens between the \texttt{<think>} and \texttt{</think>} tags before the SID output, and autoregressively decode the next SID. No rationale text is generated at inference, which removes the token overhead of explicit CoT while preserving a dedicated computation window before SID prediction.

\section{Experiments}
\label{sec:experiments}

\subsection{Experimental Setup}

\noindent\textbf{Datasets.} We evaluate on three Amazon review datasets~\citep{ni2019justifying}: Beauty, Sports and Outdoors, and Toys and Games. Following~\citep{liu2025onerec}, we filter users and items with fewer than five interactions and use a leave-last-out split: the final item is held out for testing, the second-to-last for validation, and the third-to-last as the training target with all earlier interactions as input.

\noindent\textbf{Baselines.} We compare (1) \textit{traditional sequential recommenders}: GRU4Rec~\citep{hidasi2015session}, SASRec~\citep{kang2018self}, BERT4Rec~\citep{sun2019bert4rec}, and HGN~\citep{ma2019hierarchical}; (2) \textit{generative retrieval models}: HSTU~\citep{zhai2024actions} and TIGER~\citep{rajput2023recommender}; (3) \textit{LLM-based models}: next-item SFT (our reproduction) and OneRec-Think~\citep{liu2025onerec} (explicit CoT with RLVR); and (4) \textit{implicit reasoning}: ReaRec~\citep{tang2024rearec}.

\noindent\textbf{Metrics and Implementation.} We report Hit@5, Hit@10, NDCG@5, and NDCG@10~\citep{rajput2023recommender}. The backbone is Qwen3-1.7B~\citep{qwen3technicalreport}. CPT runs for 3 epochs (lr $10^{-4}$), pause pretraining for 2 epochs (lr $10^{-3}$), and implicit SFT for 5 epochs (lr $5\times10^{-5}$) with AdamW (wd 0.01). Main results use $k{=}5$ pauses. Training and evaluation prompts match the templates in Appendix~\ref{app:pauserec_prompts}.

\subsection{Effectiveness \& Efficiency}

\begin{table}[t]
\centering
\small
\begin{tabular}{lcc}
\hline
 & \textbf{Train (GPU hours)} & \textbf{Inference (s)} \\
\hline
PauseRec & 107.34 &   0.0586 ± 0.0093\\
OneRec-Think & 305.62 & 0.2043 ± 0.0255 \\
\hline
\end{tabular}
\caption{Efficiency comparison on Qwen3-1.7B with Amazon Beauty. By avoiding RL post-training and natural-language rationale generation, \textsc{PauseRec} uses about 65\% fewer training GPU hours and is roughly 3.5$\times$ faster per inference sample than OneRec-Think.}
\label{tab:time_comparison}
\end{table}
We evaluate the effectiveness and efficiency of \textsc{PauseRec}. From Table~\ref{tab:main_results}, we observe that (1) \textbf{Consistent gains over next-item SFT:} \textsc{PauseRec} improves every metric over the next-item SFT baseline, with relative gains up to 8.85\% on Toys Hit@5. (2) \textbf{Competitive with or better than RL-based CoT:} \textsc{PauseRec} outperforms OneRec-Think on 10 of 12 metrics, including all Sports and Toys metrics, with up to 6.22\% relative improvement on Toys Hit@5; OneRec-Think remains higher on Beauty Hit@10 and NDCG@10. (3) \textbf{Substantial gains over non-LLM recommenders:} \textsc{PauseRec} consistently outperforms all baselines, highlighting the value of LLM knowledge through a SID-compatible interface.

From Table~\ref{tab:time_comparison}, \textsc{PauseRec} reduces training GPU hours by 65\% and inference latency by roughly 3.5$\times$ on Beauty by avoiding RL post-training and rationale generation. OneRec-Think uses relatively short template rationales here; inference savings grow with  generated tokens (Appendix~\ref{app:cot_sft_inference_speed}).
\subsection{Ablation Studies}

We compare our CPT-grounded pause pretraining with alternative initializations before implicit SFT: the mean of text embeddings only, the mean of SID embeddings only, and the default special-token initialization.
\begin{table}[t]
    \centering
    \small
    \begin{tabular}{lcc}
    \toprule
    \textbf{Initialization} & \textbf{Hit@5} & \textbf{NDCG@5} \\
    \midrule
    Center of text tokens only & 0.0559 & 0.0395 \\
    Center of SID tokens only & 0.0548 & 0.0387 \\
    Default Initialization & 0.0560 & 0.0394 \\
    \midrule
    \textbf{Pretrained (Ours)} & \textbf{0.0568} & \textbf{0.0401} \\
    \bottomrule
    \end{tabular}
    \caption{\texttt{<pause>} initialization ablation on Beauty. Pretraining the pause token to bridge text and SID contexts gives the best Hit@5 and NDCG@5, outperforming text-only, SID-only, and default initializations.}
    \label{tab:ablation_init}
\end{table}
Table~\ref{tab:ablation_init} shows that the pretrained \texttt{<pause>} token performs best, with modest but consistent gains over text-only, SID-only, and default initializations. This supports pause pretraining for bridging text and SID embedding spaces.

\subsection{Parameter Analysis}

We analyze the effect of pause count $k$. All settings share the same pause pretraining; we append $k$ pauses during implicit SFT and use the same $k$ at inference.
\begin{figure}[t]
  \centering
  \includegraphics[width=1\columnwidth]{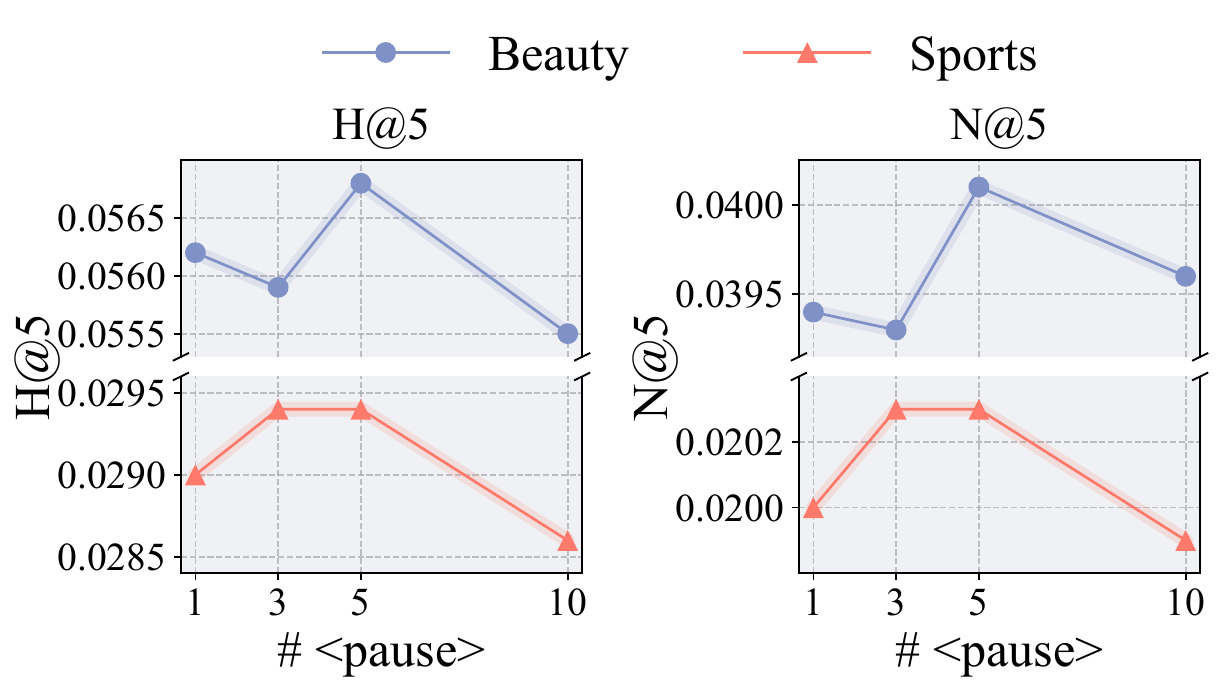}
  \caption{Effect of the number of \texttt{<pause>} tokens. Moderate latent computation works best ($k{=}5$), while further increasing $k$ provides no performance improvement.}
  \label{fig:num_pauses}
\end{figure}
Figure~\ref{fig:num_pauses} shows moderate values of $k$ work best ($k{=}5$); increasing to $k{=}10$ does not consistently help (Table~\ref{tab:pauserec_pause_tokens}), suggesting useful latent computation saturates after some pause steps.

\subsection{Qualitative Analysis}
\label{sec:analysis}

To understand how \textsc{PauseRec} uses latent pause computation during inference, we analyze where each \texttt{<pause>} token in the reasoning block attends in the surrounding context. For each pause token, we average its outgoing attention to context tokens across all layers and heads, and visualize how that distribution changes relative to the preceding pause token (Fig.~\ref{fig:attention}).
\begin{figure}[t]
  \centering
  \includegraphics[width=\columnwidth]{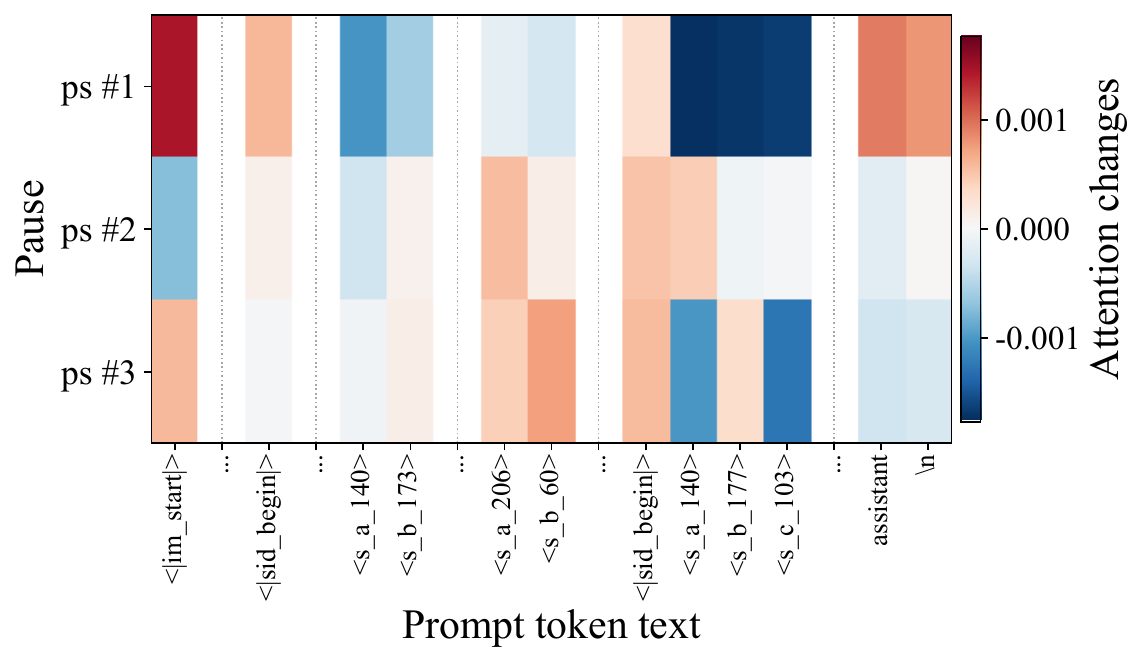}
  \caption{Attention changes across \texttt{<pause>} positions for a representative recommendation. Early pause tokens attend broadly to the prompt and history boundary, while later pause tokens focus on a smaller set of historical SID tokens; red and blue indicate increased and decreased attention after each pause.}
  \label{fig:attention}
\end{figure}
From Fig.~\ref{fig:attention} (see full prompt in Appendix~\ref{app:fig5_prompt}), we observe a multi-stage process. (1) \textbf{Context orientation:} early pauses attend broadly to the instruction and history boundary, establishing that the next SID should be inferred from purchase history. 
(2) \textbf{Preference aggregation:}
middle and later pauses shift toward historical SIDs, identifying purchases relevant to user intent. With pause position proceeds, the LLM focuses on a small salient subset of SIDs, locating items similar to target item. This staged transition explains why latent pause computation improves GR.

\section{Related Work}
\label{sec:related}



\noindent\textbf{LLM-based GR.}
Recent work uses LLMs as rankers or feature extractors~\citep{hou2023large} and as generative recommenders that output SIDs~\citep{rajput2024recommender, hua2023up5}. These pipelines typically use CPT on item-text corpora to ground SIDs~\citep{bao2023tallrec}, then apply next-item SFT. Our work builds on this foundation and asks when pretrained world knowledge improves SID prediction beyond standard training.

\noindent\textbf{Reasoning in LLMs.} Chain-of-Thought prompting improves reasoning on language tasks such as math~\citep{wei2022chain} and science~\citep{lewkowycz2022solving}, but SID-based GR involves non-linguistic outputs. Recent GR systems add CoT SFT and RL on top of CPT~\citep{liu2025onerec, yu2025thinkrec, liang2026generative}; our stage-wise analysis clarifies when those additions help. Implicit reasoning via latent tokens appears in quiet CoT~\citep{zelikman2024quiet} and ReaRec~\citep{tang2024rearec}; to our knowledge, we provide the first systematic explicit-vs-implicit comparison for LLM-based GR with an analysis of CoT SFT failure.

\section{Conclusion}
\label{sec:conclusion}

This paper shows that explicit rationales are a poor interface for SID-based generative recommendation: LLMs retain useful signals, but weakened verbalization, text--SID embedding mismatch, and rationale sensitivity limit CoT SFT. \textsc{PauseRec} replaces rationales with trainable \texttt{<pause>} tokens, enabling latent reasoning that bridges language and SIDs. Extensive experiments show that \textsc{PauseRec} is effective and efficient.
\clearpage
\bibliography{main}
\newpage
\clearpage
\appendix
\section{Limitations and Potential Risks}
\textsc{PauseRec} leaves several natural extensions for future work. First, our study uses a compact pause-token design and reports sensitivity to pause length, but does not exhaustively tune all possible pause placements, initialization schedules, or decoding variants. Second, our evaluation follows the standard offline next-item prediction protocol; complementary user-facing studies could further examine how latent reasoning affects perceived usefulness, diversity, and recommendation presentation. Finally, because implicit \texttt{<pause>} tokens are not natural-language rationales, their intermediate computation is less directly readable by users, motivating additional probing and visualization tools for analyzing how pause tokens support SID prediction.
Like other recommender systems, \textsc{PauseRec} may amplify popularity bias, reinforce historical user preferences too strongly, or inherit biases present in the interaction data and pretrained LLM. Deployments should include appropriate mitigations.

\section{AI Usage}
AI writing assistance was used only to polish grammar, clarity, and sentence flow. All research ideas, experimental designs, analyses, results, and final claims were developed, checked, and approved by the authors.

\section{Artifacts}
\label{app:artifacts}
Our artifact release includes the code and scripts for constructing SID-based training data, running the four \textsc{PauseRec} training stages, evaluating constrained SID decoding, and generating the tables and figures reported in the paper. The implementation is based on the open-source OneRec-Think repository~\citep{onerecthinkgithub}; we extend it with pause-token pretraining, implicit-reasoning finetuning, evaluation utilities, and experiment orchestration for \textsc{PauseRec}. OneRec-Think is released under the Apache License 2.0, which permits reuse and modification. Our code is released under the MIT License.

\section{Theoretical Analysis of Text--SID Separation}
\label{app:text_sid_theorem}
We formalize why geometric separation between natural-language and SID representations can weaken explicit CoT. Consider the final step after a rationale, where the model must generate a SID token. Let $v_s$ denote the output embedding of SID token $s$, and let the SID logit be
\[
z_s(h)=v_s^\top h,
\]
where $h$ is the hidden state before SID generation. Let $\mathcal{U}_{\text{text}}$ be the subspace in which hidden states move when the model generates or is optimized on natural-language rationale tokens, and let
\[
\mathcal{U}_{\text{SID}}=\mathrm{span}\{v_y-v_s: y,s\in\mathcal{S}\}
\]
be the subspace that controls relative SID logits. Define the text--SID coupling coefficient
\[
\rho = \|P_{\mathcal{U}_{\text{SID}}}P_{\mathcal{U}_{\text{text}}}\|_2,
\]
where $P_{\mathcal{U}}$ is the orthogonal projection onto subspace $\mathcal{U}$. Smaller $\rho$ means stronger separation between text-induced hidden-state movement and SID-discriminative directions.

\begin{theorem}[Text--SID separation bounds the effect of verbal rationales]
Suppose adding a natural-language rationale changes the hidden state before SID generation from $h$ to $h+\Delta$, where
\[
\Delta = \Delta_{\text{text}} + r,\quad
\Delta_{\text{text}}\in\mathcal{U}_{\text{text}},\quad
\|r\|\le \epsilon.
\]
Assume $\|v_y-v_s\|\le B$ for all valid SID tokens $y,s\in\mathcal{S}$. Let
\[
\rho = \|P_{\mathcal{U}_{\mathrm{SID}}}P_{\mathcal{U}_{\text{text}}}\|_2.
\]
Then for any target SID token $y$ and competing SID token $s$,
\begin{equation}
\begin{aligned}
    &\left|
\bigl(z_y(h+\Delta)-z_s(h+\Delta)\bigr)
-
\bigl(z_y(h)-z_s(h)\bigr)
\right|\\
&\le
B(\rho\|\Delta_{\text{text}}\|+\epsilon).
\end{aligned}
\end{equation}

Consequently, if the target SID initially trails some competitor by margin $\gamma$,
\[
z_y(h)-z_s(h)\le -\gamma,
\]
and
\[
\gamma > B(\rho\|\Delta_{\text{text}}\|+\epsilon),
\]
then the rationale cannot make $y$ outrank $s$.
\end{theorem}

\begin{proof}
Let the SID margin between target token $y$ and competing token $s$ be
\[
m(h)=z_y(h)-z_s(h).
\]
The change in this margin after adding the rationale is
\begin{equation}
\begin{aligned}
m(h+\Delta)-m(h)
=
\bigl(z_y(h+\Delta)-z_s(h+\Delta)\bigr)\\
-
\bigl(z_y(h)-z_s(h)\bigr).
\end{aligned}
\end{equation}
Since $z_s(h)=v_s^\top h$, we have
\[
m(h+\Delta)-m(h)
=
(v_y-v_s)^\top \Delta.
\]
Substituting $\Delta=\Delta_{\text{text}}+r$ gives
\[
|(v_y-v_s)^\top \Delta|
\le
|(v_y-v_s)^\top \Delta_{\text{text}}|
+
|(v_y-v_s)^\top r|.
\]
Because only the projection of $\Delta_{\text{text}}$ onto the SID-discriminative subspace can affect relative SID logits,
\[
|(v_y-v_s)^\top \Delta_{\text{text}}|
\le
\|v_y-v_s\|\cdot \|P_{\mathcal{U}_{\mathrm{SID}}}\Delta_{\text{text}}\|.
\]
Since $\Delta_{\text{text}}\in\mathcal{U}_{\text{text}}$,
\[
P_{\mathcal{U}_{\mathrm{SID}}}\Delta_{\text{text}}
=
P_{\mathcal{U}_{\mathrm{SID}}}P_{\mathcal{U}_{\text{text}}}\Delta_{\text{text}}.
\]
By the definition of $\rho$,
\[
\|P_{\mathcal{U}_{\mathrm{SID}}}P_{\mathcal{U}_{\text{text}}}\Delta_{\text{text}}\|
\le
\rho\|\Delta_{\text{text}}\|.
\]
Therefore,
\[
|(v_y-v_s)^\top \Delta_{\text{text}}|
\le
B\rho\|\Delta_{\text{text}}\|.
\]
For the residual term,
\[
|(v_y-v_s)^\top r|
\le
\|v_y-v_s\|\|r\|
\le
B\epsilon.
\]
Combining the two bounds yields
\[
|m(h+\Delta)-m(h)|
\le
B(\rho\|\Delta_{\text{text}}\|+\epsilon).
\]

It remains to show the ranking consequence. Define
\[
M = B(\rho\|\Delta_{\text{text}}\|+\epsilon).
\]
The bound above implies that the rationale can change the SID margin by at most $M$. If the target SID initially trails competitor $s$ by margin $\gamma$, then
\[
m(h)=z_y(h)-z_s(h)\le -\gamma.
\]
After adding the rationale,
\begin{align*}
m(h+\Delta)
&=
m(h)+\bigl(m(h+\Delta)-m(h)\bigr)\\
&\le
m(h)+M.
\end{align*}
Since $m(h)\le -\gamma$, we obtain
\[
m(h+\Delta)\le -\gamma+M.
\]
If $\gamma>M$, then
\[
m(h+\Delta)<0.
\]
Equivalently,
\[
z_y(h+\Delta)<z_s(h+\Delta).
\]
Thus, even after adding the rationale, the target SID token $y$ still receives a lower logit than the competing SID token $s$, so $y$ cannot outrank $s$.
\end{proof}


\section{\textsc{PauseRec} Embedding Visualization}
\label{app:pauserec_embedding_viz}

\begin{figure*}[t]
    \centering
    \includegraphics[width=0.9\linewidth]{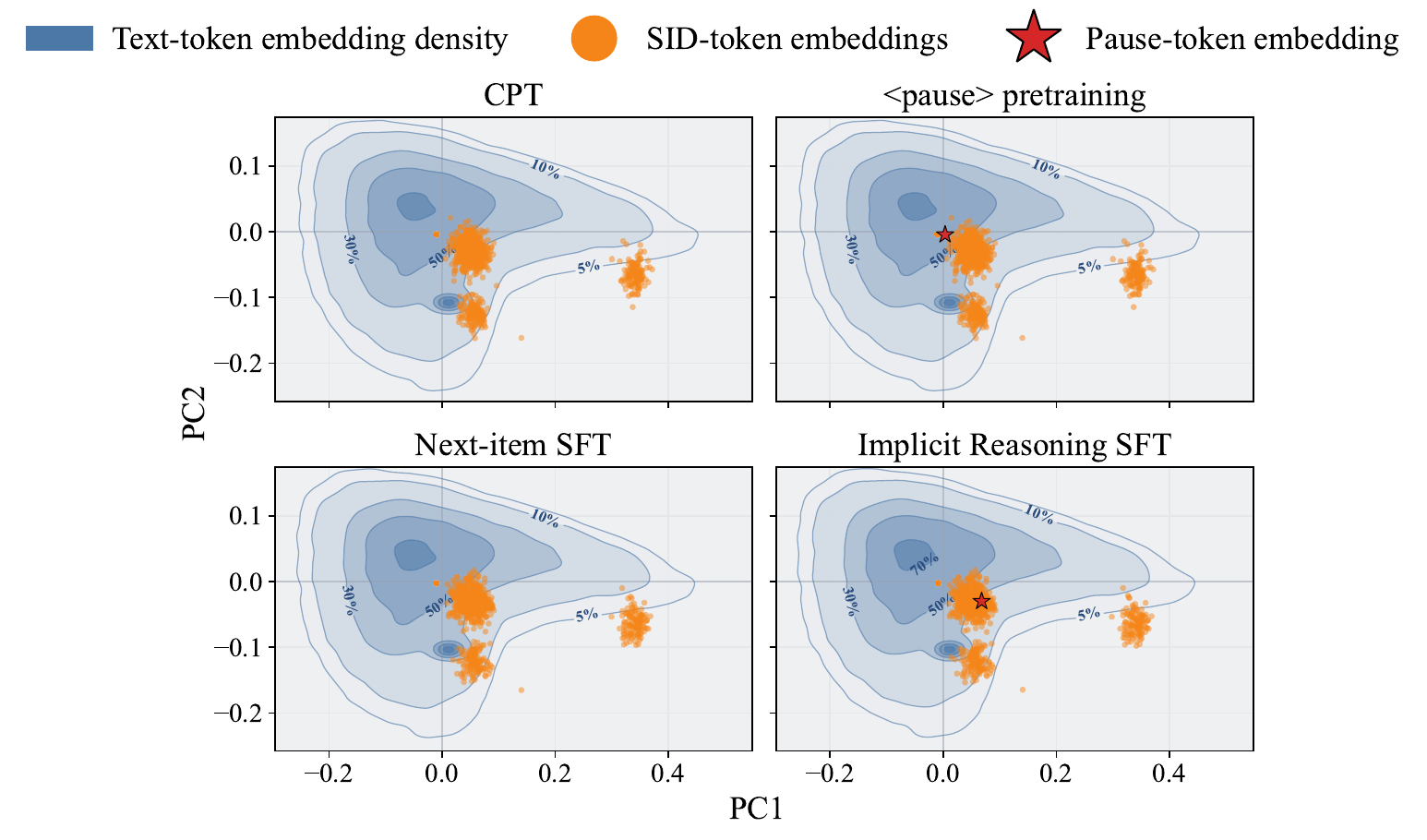}
    \caption{Token embeddings after each stage of the \textsc{PauseRec} pipeline. The learned \texttt{<pause>} token lies near the boundary between natural-language tokens and SID tokens, indicating that pause pretraining positions it as a bridge between the two embedding spaces. This supports the role of \texttt{<pause>} tokens in connecting semantic information from natural language to SID prediction.}
    \label{fig:pauserec_embedding_with_pause}
\end{figure*}

Figure~\ref{fig:pauserec_embedding_with_pause} visualizes token embeddings after each stage of the \textsc{PauseRec} pipeline, including CPT, pause-token pretraining, next-item SFT, and implicit-reasoning SFT. Across stages, the \texttt{<pause>} token stays at the boundary between the natural-language token cluster and the SID token cluster rather than collapsing into either group. This boundary placement provides empirical evidence that the pause token connects semantics across the two embedding spaces, helping route natural-language knowledge toward SID generation and thereby improving recommendation performance.

\section{Implementation Details}
\label{app:implementation_details}
\label{app:details}

\subsection{Sample Prompts for \textsc{PauseRec} Training and Inference}
\label{app:pauserec_prompts}

We provide concrete prompt text for each stage of \textsc{PauseRec} on Amazon Beauty, using the same leave-last-out training split as in Section~\ref{sec:experiments}. The example user has two items in the prediction history; the held-out target item is \emph{Raw African Black Soap from Ghana 1 Lb} (semantic ID shown in the implicit-reasoning blocks below). Chat-based stages use the fixed system instruction and chat turn delimiters shown in the full-sequence examples. During pause pretraining, \texttt{<pause>} tokens are inserted at random word boundaries. During implicit-reasoning SFT and inference, $k{=}5$ pause tokens are placed between the \texttt{<think>} and \texttt{</think>} tags before the target SID; loss is masked on those pause positions during SFT.

\noindent\textit{Continual pretraining (CPT).}
\begin{promptbox}
The user has purchased the following items:
<|sid\_begin|>\allowbreak{}<s\_a\_99>\allowbreak{}<s\_b\_19>\allowbreak{}<s\_c\_220>\allowbreak{}<s\_d\_204>\allowbreak{}<|sid\_end|>,
its title is "Phyto Phytocitrus Restructuring Mask
for Unisex, 6.7 Ounce", its categories are "Beauty >
Hair Care > Conditioners";
<|sid\_begin|>\allowbreak{}<s\_a\_238>\allowbreak{}<s\_b\_74>\allowbreak{}<s\_c\_13>\allowbreak{}<s\_d\_122>\allowbreak{}<|sid\_end|>,
its title is "Matrix Biolage Colorcaretherapie Color
Care Shampoo and Conditioner Set 33.8oz 1 Liter",
its categories are "Beauty > Hair Care > Shampoo \&
Conditioner Sets";
<|sid\_begin|>\allowbreak{}<s\_a\_226>\allowbreak{}<s\_b\_110>\allowbreak{}<s\_c\_129>\allowbreak{}<s\_d\_207>\allowbreak{}<|sid\_end|>,
its title is "Raw African Black Soap from Ghana 1
Lb", its categories are "Beauty > Bath \& Body >
Cleansers > Soaps";
\end{promptbox}

\noindent\textit{Pause-token pretraining (10\% random <pause> insertion; seed 42).}
\begin{promptbox}
The user has <pause> purchased the following items:
<|sid\_begin|>\allowbreak{}<s\_a\_99>\allowbreak{}<s\_b\_19>\allowbreak{}<s\_c\_220>\allowbreak{}<s\_d\_204>\allowbreak{}<|sid\_end|>,
its title is "Phyto <pause> <pause> Phytocitrus
Restructuring <pause> Mask for Unisex, 6.7 Ounce",
its categories are "Beauty > <pause> Hair Care
<pause> > Conditioners";
<|sid\_begin|>\allowbreak{}<s\_a\_238>\allowbreak{}<s\_b\_74>\allowbreak{}<s\_c\_13>\allowbreak{}<s\_d\_122>\allowbreak{}<|sid\_end|>,
<pause> its title is "Matrix Biolage
Colorcaretherapie Color Care Shampoo and Conditioner
Set 33.8oz 1 Liter", its categories are "Beauty >
Hair Care > Shampoo \& Conditioner Sets";
<|sid\_begin|>\allowbreak{}<s\_a\_226>\allowbreak{}<s\_b\_110>\allowbreak{}<s\_c\_129>\allowbreak{}<s\_d\_207>\allowbreak{}<|sid\_end|>,
its title is "Raw African Black Soap from Ghana 1
Lb", its categories are "Beauty > Bath <pause> \&
Body > Cleansers > Soaps";
\end{promptbox}

\noindent\textit{Next-item supervised finetuning: user prompt.}
\begin{promptbox}
The user has purchased the following items:
<|sid\_begin|>\allowbreak{}<s\_a\_99>\allowbreak{}<s\_b\_19>\allowbreak{}<s\_c\_220>\allowbreak{}<s\_d\_204>\allowbreak{}<|sid\_end|>;
<|sid\_begin|>\allowbreak{}<s\_a\_238>\allowbreak{}<s\_b\_74>\allowbreak{}<s\_c\_13>\allowbreak{}<s\_d\_122>\allowbreak{}<|sid\_end|>;
\end{promptbox}

\noindent\textit{Next-item supervised finetuning: full sequence (empty thinking block).}
\begin{promptbox}
<|im\_start|>system
You are a professional recommendation expert who
needs to recommend the next possible purchase for
users based on their purchase history. Please
predict the most likely next product that the user
will purchase based on the user's historical
purchase information.<|im\_end|>
<|im\_start|>user
The user has purchased the following items:
<|sid\_begin|>\allowbreak{}<s\_a\_99>\allowbreak{}<s\_b\_19>\allowbreak{}<s\_c\_220>\allowbreak{}<s\_d\_204>\allowbreak{}<|sid\_end|>;
<|sid\_begin|>\allowbreak{}<s\_a\_238>\allowbreak{}<s\_b\_74>\allowbreak{}<s\_c\_13>\allowbreak{}<s\_d\_122>\allowbreak{}<|sid\_end|>;<|im\_end|>
<|im\_start|>assistant
<think>

</think>
<|sid\_begin|>\allowbreak{}<s\_a\_226>\allowbreak{}<s\_b\_110>\allowbreak{}<s\_c\_129>\allowbreak{}<s\_d\_207>\allowbreak{}<|sid\_end|>\allowbreak{}<|im\_end|>
\end{promptbox}

\noindent\textit{Implicit-reasoning finetuning: user prompt and target SID.}
\begin{promptbox}
The user has purchased the following items:
<|sid\_begin|>\allowbreak{}<s\_a\_99>\allowbreak{}<s\_b\_19>\allowbreak{}<s\_c\_220>\allowbreak{}<s\_d\_204>\allowbreak{}<|sid\_end|>,
its title is "Phyto Phytocitrus Restructuring Mask
for Unisex, 6.7 Ounce", its categories are "Beauty >
Hair Care > Conditioners";
<|sid\_begin|>\allowbreak{}<s\_a\_238>\allowbreak{}<s\_b\_74>\allowbreak{}<s\_c\_13>\allowbreak{}<s\_d\_122>\allowbreak{}<|sid\_end|>,
its title is "Matrix Biolage Colorcaretherapie Color
Care Shampoo and Conditioner Set 33.8oz 1 Liter",
its categories are "Beauty > Hair Care > Shampoo \&
Conditioner Sets";
Target SID:
<|sid\_begin|>\allowbreak{}<s\_a\_226>\allowbreak{}<s\_b\_110>\allowbreak{}<s\_c\_129>\allowbreak{}<s\_d\_207>\allowbreak{}<|sid\_end|>
\end{promptbox}

\noindent\textit{Implicit-reasoning finetuning: full sequence ($k{=}5$ pause tokens).}
\begin{promptbox}
<|im\_start|>system
You are a professional recommendation expert who
needs to recommend the next possible purchase for
users based on their purchase history. Please
predict the most likely next product that the user
will purchase based on the user's historical
purchase information.<|im\_end|>
<|im\_start|>user
The user has purchased the following items:
<|sid\_begin|>\allowbreak{}<s\_a\_99>\allowbreak{}<s\_b\_19>\allowbreak{}<s\_c\_220>\allowbreak{}<s\_d\_204>\allowbreak{}<|sid\_end|>,
its title is "Phyto Phytocitrus Restructuring Mask
for Unisex, 6.7 Ounce", its categories are "Beauty >
Hair Care > Conditioners";
<|sid\_begin|>\allowbreak{}<s\_a\_238>\allowbreak{}<s\_b\_74>\allowbreak{}<s\_c\_13>\allowbreak{}<s\_d\_122>\allowbreak{}<|sid\_end|>,
its title is "Matrix Biolage Colorcaretherapie Color
Care Shampoo and Conditioner Set 33.8oz 1 Liter",
its categories are "Beauty > Hair Care > Shampoo \&
Conditioner Sets";<|im\_end|>
<|im\_start|>assistant
<think>
<pause>\allowbreak{}<pause>\allowbreak{}<pause>\allowbreak{}<pause>\allowbreak{}<pause>
</think>
<|sid\_begin|>\allowbreak{}<s\_a\_226>\allowbreak{}<s\_b\_110>\allowbreak{}<s\_c\_129>\allowbreak{}<s\_d\_207>\allowbreak{}<|sid\_end|>\allowbreak{}<|im\_end|>
\end{promptbox}

\noindent\textit{Inference: prompt prefix before constrained SID decoding.}
\begin{promptbox}
<|im\_start|>system
You are a professional recommendation expert who
needs to recommend the next possible purchase for
users based on their purchase history. Please
predict the most likely next product that the user
will purchase based on the user's historical
purchase information.<|im\_end|>
<|im\_start|>user
The user has purchased the following items:
<|sid\_begin|>\allowbreak{}<s\_a\_99>\allowbreak{}<s\_b\_19>\allowbreak{}<s\_c\_220>\allowbreak{}<s\_d\_204>\allowbreak{}<|sid\_end|>,
its title is "Phyto Phytocitrus Restructuring Mask
for Unisex, 6.7 Ounce", its categories are "Beauty >
Hair Care > Conditioners";
<|sid\_begin|>\allowbreak{}<s\_a\_238>\allowbreak{}<s\_b\_74>\allowbreak{}<s\_c\_13>\allowbreak{}<s\_d\_122>\allowbreak{}<|sid\_end|>,
its title is "Matrix Biolage Colorcaretherapie Color
Care Shampoo and Conditioner Set 33.8oz 1 Liter",
its categories are "Beauty > Hair Care > Shampoo \&
Conditioner Sets";<|im\_end|>
<|im\_start|>assistant
<think>
<pause>\allowbreak{}<pause>\allowbreak{}<pause>\allowbreak{}<pause>\allowbreak{}<pause>
</think>
\end{promptbox}

\subsection{Prompt for Qualitative Attention Analysis}
\label{app:fig5_prompt}

Figure~\ref{fig:attention} visualizes the pause-token attention pattern for the following example. The target item is a Dove hair styling spray, and the history contains multiple hair-care and styling products. In the later pause steps, attention to the history item whose SID begins with <s\_a\_206>\allowbreak{}<s\_b\_60>, \emph{Natures Bounty Optimal Solutions Hair, Skin and Nails Gummies}, increases. This item is related to the target through hair-care intent, so the increase supports the main-paper analysis that later pauses retrieve and aggregate target-relevant historical evidence before SID generation.

Target Item
Title: Dove Hair Styling Oxygen Moisture Root Lift Spray, 3.3 Ounce
SID: <|sid\_begin|>\allowbreak{}<s\_a\_140>\allowbreak{}<s\_b\_39>\allowbreak{}<s\_c\_151>\allowbreak{}<s\_d\_68>\allowbreak{}<|sid\_end|>
Categories: Beauty > Hair Care > Styling Products > Hair Sprays

\noindent\textit{Full prompt text.}
\begin{promptbox}
<|im\_start|>system
You are a professional recommendation expert who needs to recommend the next
possible purchase for users based on their purchase history. Please predict the
most likely next product that the user will purchase based on the user's
historical purchase information.<|im\_end|>
<|im\_start|>user
The user has purchased the following items:
<|sid\_begin|>\allowbreak{}<s\_a\_6>\allowbreak{}<s\_b\_192>\allowbreak{}<s\_c\_205>\allowbreak{}<s\_d\_33>\allowbreak{}<|sid\_end|>,
its title is "Axe Primed Just Clean Shampoo, 12-Ounce Bottle (Pack of 3)",
its categories are "Beauty > Hair Care > Shampoos";
<|sid\_begin|>\allowbreak{}<s\_a\_248>\allowbreak{}<s\_b\_8>\allowbreak{}<s\_c\_99>\allowbreak{}<s\_d\_150>\allowbreak{}<|sid\_end|>,
its title is "Olay Regenerist Micro-Sculpting Serum 1.7 Fl Oz",
its categories are "Beauty > Skin Care > Face";
<|sid\_begin|>\allowbreak{}<s\_a\_113>\allowbreak{}<s\_b\_56>\allowbreak{}<s\_c\_77>\allowbreak{}<s\_d\_2>\allowbreak{}<|sid\_end|>,
its title is "Clearasil Ultra Acne Treatment Daily Face Wash, 6.78 Ounce (Pack of 3)",
its categories are "Beauty > Skin Care > Face > Cleansers > Washes";
<|sid\_begin|>\allowbreak{}<s\_a\_6>\allowbreak{}<s\_b\_6>\allowbreak{}<s\_c\_17>\allowbreak{}<s\_d\_210>\allowbreak{}<|sid\_end|>,
its title is "Pantene Pro-V Expert Collection Agedefy Conditioner 8.4 Fl Oz",
its categories are "Beauty > Hair Care > Conditioners";
<|sid\_begin|>\allowbreak{}<s\_a\_6>\allowbreak{}<s\_b\_222>\allowbreak{}<s\_c\_222>\allowbreak{}<s\_d\_71>\allowbreak{}<|sid\_end|>,
its title is "Pantene Pro-V Expert Collection Agedefy Shampoo 10.1 Fl Oz",
its categories are "Beauty > Hair Care > Shampoos";
<|sid\_begin|>\allowbreak{}<s\_a\_255>\allowbreak{}<s\_b\_71>\allowbreak{}<s\_c\_242>\allowbreak{}<s\_d\_99>\allowbreak{}<|sid\_end|>,
its title is "Burt's Bees Lip Gloss, Autumn Haze, 0.2 Fluid Ounces",
its categories are "Beauty > Makeup > Lips > Lip Glosses";
<|sid\_begin|>\allowbreak{}<s\_a\_155>\allowbreak{}<s\_b\_51>\allowbreak{}<s\_c\_96>\allowbreak{}<s\_d\_246>\allowbreak{}<|sid\_end|>,
its title is "Nexxus Youth Renewal Rejuvenating Shampoo, 13.5 Ounce",
its categories are "Beauty > Hair Care > Shampoos";
<|sid\_begin|>\allowbreak{}<s\_a\_248>\allowbreak{}<s\_b\_86>\allowbreak{}<s\_c\_54>\allowbreak{}<s\_d\_216>\allowbreak{}<|sid\_end|>,
its title is "Simple Protecting Light Moisturizer Spf 15, 4.2 Ounce",
its categories are "Beauty > Skin Care > Face > Creams \& Moisturizers >
Fluids \& Lotions > Lotions";
<|sid\_begin|>\allowbreak{}<s\_a\_113>\allowbreak{}<s\_b\_255>\allowbreak{}<s\_c\_31>\allowbreak{}<s\_d\_115>\allowbreak{}<|sid\_end|>,
its title is "Dove go fresh, Burst Body Wash, 24 Ounce (Pack of 2)",
its categories are "Beauty > Bath \& Body > Cleansers > Body Washes";
<|sid\_begin|>\allowbreak{}<s\_a\_140>\allowbreak{}<s\_b\_254>\allowbreak{}<s\_c\_162>\allowbreak{}<s\_d\_130>\allowbreak{}<|sid\_end|>,
its title is "Nexxus Youth Renewal Plump and Lift Blow Dry Spray, 7.5 Ounce",
its categories are "Beauty > Hair Care > Styling Products > Hair Sprays";
<|sid\_begin|>\allowbreak{}<s\_a\_140>\allowbreak{}<s\_b\_205>\allowbreak{}<s\_c\_36>\allowbreak{}<s\_d\_36>\allowbreak{}<|sid\_end|>,
its title is "Nexxus Youth Renewal Rejuvenating Elixir, 0.94 Ounce",
its categories are "Beauty > Hair Care > Hair \& Scalp Treatments";
<|sid\_begin|>\allowbreak{}<s\_a\_155>\allowbreak{}<s\_b\_190>\allowbreak{}<s\_c\_158>\allowbreak{}<s\_d\_81>\allowbreak{}<|sid\_end|>,
its title is "CLEAR MEN SCALP THERAPY 2 in 1 AntiDandruff Shampoo and Conditioner,
Dry Scalp Hydration, 12.9oz", its categories are "Beauty > Hair Care > Shampoos";
<|sid\_begin|>\allowbreak{}<s\_a\_113>\allowbreak{}<s\_b\_67>\allowbreak{}<s\_c\_140>\allowbreak{}<s\_d\_115>\allowbreak{}<|sid\_end|>,
its title is "Schick Hydro Silk Disposable Razor, 3 Count",
its categories are "Beauty > Skin Care > Body > Moisturizers > Oils";
<|sid\_begin|>\allowbreak{}<s\_a\_21>\allowbreak{}<s\_b\_45>\allowbreak{}<s\_c\_114>\allowbreak{}<s\_d\_49>\allowbreak{}<|sid\_end|>,
its title is "Own Products Refining Moisture Night Cream",
its categories are "Beauty > Skin Care > Face > Creams \& Moisturizers > Night Creams";
<|sid\_begin|>\allowbreak{}<s\_a\_113>\allowbreak{}<s\_b\_204>\allowbreak{}<s\_c\_185>\allowbreak{}<s\_d\_233>\allowbreak{}<|sid\_end|>,
its title is "Nivea Q10 Skin Firming Body Lotion , 13.5 fl oz  (Pack of 2)",
its categories are "Beauty > Skin Care > Body > Moisturizers > Lotions";
<|sid\_begin|>\allowbreak{}<s\_a\_135>\allowbreak{}<s\_b\_169>\allowbreak{}<s\_c\_250>\allowbreak{}<s\_d\_60>\allowbreak{}<|sid\_end|>,
its title is "L'Oreal Paris Age Perfect Hydra-Nutrition Moisturizer, 1.7-Fluid Ounce",
its categories are "Beauty > Skin Care > Face > Creams \& Moisturizers >
Fluids \& Lotions > Fluids";
<|sid\_begin|>\allowbreak{}<s\_a\_238>\allowbreak{}<s\_b\_3>\allowbreak{}<s\_c\_5>\allowbreak{}<s\_d\_18>\allowbreak{}<|sid\_end|>,
its title is "Cristophe Professional Glossing Shampoo, 10 Ounce",
its categories are "Beauty > Hair Care > Shampoos";
<|sid\_begin|>\allowbreak{}<s\_a\_140>\allowbreak{}<s\_b\_173>\allowbreak{}<s\_c\_52>\allowbreak{}<s\_d\_91>\allowbreak{}<|sid\_end|>,
its title is "Tresemme Keratin Smooth Smoothing Creme Serum, 3.5 Ounce",
its categories are "Beauty > Hair Care > Styling Products > Creams, Gels \& Lotions";
<|sid\_begin|>\allowbreak{}<s\_a\_206>\allowbreak{}<s\_b\_60>\allowbreak{}<s\_c\_224>\allowbreak{}<s\_d\_48>\allowbreak{}<|sid\_end|>,
its title is "Natures Bounty Optimal Solutions Hair, Skin and Nails Gummies, 80 Count",
its categories are "Beauty > Skin Care";
<|sid\_begin|>\allowbreak{}<s\_a\_140>\allowbreak{}<s\_b\_177>\allowbreak{}<s\_c\_103>\allowbreak{}<s\_d\_200>\allowbreak{}<|sid\_end|>,
its title is "Dove Hair Styling Oxygen Moisture Leave In Foam, 5.1 Ounce",
its categories are "Beauty > Hair Care > Styling Products > Mousses \& Foams";<|im\_end|>
<|im\_start|>assistant
<think>
<pause>\allowbreak{}<pause>\allowbreak{}<pause>
</think>
\end{promptbox}

\subsection{SID Metadata Decoding Prompt}
\label{app:sid_decode_prompts}
For the metadata recovery experiment in Table~\ref{tab:metadata}, we use the first two items in each dataset's pretraining file as in-context examples and then query the remaining items. No judge model is used; predictions are evaluated by exact string match after parsing the generated Title: and Category: fields. The prompt template is shown below, with the final user turn line-wrapped for readability:
\begin{promptbox}
<|im\_start|>system
Please generate the title and category of the product based on its semantic ID.<|im\_end|>
<|im\_start|>user
\{shot\_1\_sid\}<|im\_end|
<|im\_start|>assistant
Title: "\{shot\_1\_title\}", Category: "\{shot\_1\_categories\}"
<|im\_end|>
<|im\_start|>user
\{shot\_2\_sid\}<|im\_end|
<|im\_start|>assistant
Title: "\{shot\_2\_title\}", Category: "\{shot\_2\_categories\}"
<|im\_end|>
<|im\_start|>user
Generate the Title and Category of this product \{query\_sid\}.
Please only generate with NO thinking! <|im\_end|>
<|im\_start|>assistant
\end{promptbox}

Generation uses greedy decoding (do\_sample=False) with max\_new\_tokens=256.

\subsection{Sample Rationales Used in CoT SFT}
Below we provide the reasoning field from the first sample of each CoT SFT variant.

\noindent\textbf{Template-Category.}
\begin{promptbox}
The user is likely to buy items in the Beauty $>$ Tools \& Accessories $>$ Mirrors $>$ Makeup Mirrors category.
\end{promptbox}

\noindent\textbf{Template-Extended.}
\begin{promptbox}
The user demonstrates interest in hair care and skin care products. By identifying the user's preference in natural, high-quality beauty and personal care items, we can predict that, to complement their previous purchases and enhance their daily grooming routine, the user will purchase beauty tools, for example, a Home Travel 9X/1X Folding Lighted Cosmetic Mirror.
\end{promptbox}

\noindent\textbf{Gemini 3.1 Flash-Lite Free-form.}
\begin{promptbox}
The user consistently purchases hair care products and beauty tools, suggesting a strong interest in hair maintenance and styling. Given the previous acquisition of hair care treatments and a cosmetic mirror, the model can infer a logical progression toward purchasing functional hair styling accessories like clips to complement their existing routine.
\end{promptbox}

\noindent\textbf{Gemini 3.1 Pro Free-form.}
\begin{promptbox}
The user's purchase history demonstrates a strong focus on hair care and beauty maintenance, moving from washing and conditioning treatments to grooming tools like a cosmetic mirror. An LLM would logically infer that after acquiring products to clean and treat their hair, the user's next step in their routine would be purchasing styling accessories like hair clips to manage and style it.
\end{promptbox}

\noindent\textbf{Gemini 3.1 Flash-Lite Rejection.}
\begin{promptbox}
The user consistently purchases hair care products and beauty tools, indicating a strong focus on hair maintenance and grooming. By identifying the pattern of hair-related purchases, the model can infer a logical progression toward styling accessories like hair clips to complement the existing hair care routine.
\end{promptbox}

\noindent\textbf{Gemini 3.1 FL Gemini Rejection.}
\begin{promptbox}
The user consistently invests in high-quality, natural, and specialized personal grooming products, suggesting a transition from purchasing consumable maintenance items to acquiring tools that facilitate their beauty and self-care routine. A lighted cosmetic mirror serves as a logical functional upgrade to complement their established regimen of premium hair and skin care products.
\end{promptbox}

\noindent\textbf{Gemini 3.1 Flash-Lite Restricted.}
\begin{promptbox}
Step 1: The user's purchase history shows a consistent focus on beauty and personal care products. Step 2: This suggests the user is expanding from consumables to grooming tools. Step 3: A lighted cosmetic mirror is the most logical next purchase because it supports the user's existing routine.
\end{promptbox}

\subsection{Teacher Model Prompts for Reasoning Generation in CoT SFT}
\label{app:prompts}

\subsubsection{Free-form Reasoning Prompt}
\begin{promptbox}
Given a user's purchase history, explain step-by-step how an LLM might reason to predict the next item they would purchase.

User's purchase history: \{description\}

The ground truth next item is:
- Title: "\{groundtruth\_title\}"
- Categories: "\{groundtruth\_categories\}"

Provide concise reasoning (1-2 sentences) explaining how an LLM could logically infer this next purchase from the user's prior items. Focus on patterns in categories, titles, or user preferences that would lead to this prediction. Do not include any preamble or labels—output only the reasoning text.
\end{promptbox}

\subsubsection{Format-Restricted Reasoning Prompt}
\begin{promptbox}
P1: System Role \& Task Definition
You are an expert at analyzing e-commerce purchase patterns and predicting user preferences.
You are helping create reasoning traces for recommendation training data.
You will receive a user's purchase history, a list of candidate items, and one known positive target item.
Your goal is to produce a predictive rationale: first infer the user's most likely next need from the purchase history, then verify whether the known target item is the best available match.

P2: Collaborative Context Presentation
=== USER PURCHASE HISTORY ===
\{history\_block\}

=== CANDIDATE ITEMS ===
\{candidate\_block\}

=== KNOWN TARGET ITEM ===
Candidate \{target\_candidate\_id\}: Title: \{target\_record["title"]\}; Categories: \{target\_record["categories"]\}; SID: \{target\_record["sid"]\}

P3: Reasoning Procedure
Work in two phases.
Phase A: Before relying on the known target item, identify the purchase-history items with the strongest evidence. Cite the smallest number of SIDs needed to support the inferred pattern, usually 1-4. Based on those items, infer the user's likely next need, replenishment signal, complement need, routine continuation, tool-versus-consumable transition, or category progression.
Phase B: Then evaluate whether the known target item matches that inferred need better than the strongest one or two alternatives in the candidate list.
Prefer concrete signals such as repetition, recency, category progression, complementarity, brand continuity, and tool-versus-consumable transitions when they are supported by the evidence.

P4: Critical Guidelines as Output Constraint
CRITICAL GUIDELINES:
1. When referring to purchase-history items, cite them directly using their SID.
2. In Step 1 and Step 2, do not mention the target title, target SID, or candidate numbers.
3. Use only evidence visible in the provided titles and categories.
4. Do not hallucinate hidden preferences or unsupported product attributes.
5. In Step 3, compare against at most two real alternative candidates; do not discuss the full candidate list.
6. If the target is not strongly supported, say "the fit is weaker than ideal but still the best available match" instead of inventing certainty.
7. Keep the rationale concise, specific, and information-dense. Avoid generic filler.
8. Do not output markdown fences.

P5: Structured Multi-Step Reasoning Format
Return strict JSON with a single key "reasoning".
The value must be one string in this format:
Step 1: History-only evidence summary citing the most relevant SIDs and the dominant recurring pattern.
Step 2: Based only on the history, infer the user's most likely next need or next category, and state confidence as high, medium, or low.
Step 3: Now evaluate why the target item "\{target\_record["title"]\}" is the best available match to that inferred need relative to the strongest one or two alternatives.
Summary: One short concluding sentence, under 20 words.
\end{promptbox}

\subsection{General Language Benchmark Prompts and Generated Reasoning}
\label{app:general_language_prompts}

For the diagnostic experiment in Table~\ref{tab:forgetting}, we evaluated the target model on MMLU, HellaSwag, PIQA, and ARC-Challenge using a general-language benchmark run. 

\noindent\textbf{Text-generation prompt.}
For the target model's text generation, we used a chat-style prompt that explicitly opens a thinking block:
\begin{promptbox}
<|im\_start|>user
Answer the following multiple-choice question. You may explain briefly if useful, but finish with 'Final answer: <letter>'.
\{multiple-choice block\}
<|im\_end|>
<|im\_start|>assistant
<think>
\end{promptbox}
The multiple-choice block lists the question followed by lettered options and ends with Final answer:. For MMLU only, the block contains five in-context examples sampled from the MMLU development split before the test question; the other benchmarks are zero-shot. HellaSwag questions are prefixed with Complete the sentence:. PIQA uses the two candidate solutions as options A/B. ARC-Challenge uses the answer choices from the dataset after normalizing them to A/B/C/D labels.

\noindent\textbf{Logit prompt.}
For logit-based accuracy, we did not ask the model to generate a rationale. Instead, we used the following prompt and compared the next-token logits assigned to the option letters:
\begin{promptbox}
Answer the following multiple-choice question with the option letter only.
\{multiple-choice block ending with Answer:\}
\end{promptbox}
MMLU again uses five in-context examples, each ending with Answer: \{gold letter\}, followed by the test question ending with Answer:. HellaSwag, PIQA, and ARC-Challenge use the same zero-shot question blocks as above but end with Answer:.

\noindent\textbf{What the model actually generated.} Text generations used greedy decoding (do\_sample=False) with max\_new\_tokens=128. The extracted reasoning\_text field was non-empty for all 27,094 target examples. However, the generated text was usually not a multi-step rationale. It was most often a short answer-likelihood statement inside the \texttt{<think>} block, followed after \texttt{</think>} by SID tokens. For example, the first ARC-Challenge row has raw generation:
\begin{promptbox}
The user is likely to answer C
</think>
<|sid\_begin|>\allowbreak{}<s\_a\_91>\allowbreak{}<s\_b\_84>
<s\_c\_156>\allowbreak{}<s\_d\_20>\allowbreak{}<|sid\_end|>\allowbreak{}<|im\_end|>
\end{promptbox}
Thus, the content inside \texttt{<think>...</think>} for that example is exactly:
\begin{promptbox}
The user is likely to answer C.
\end{promptbox}
Representative extracted reasoning\_text entries from the artifact are:
\begin{itemize}
    \item MMLU: The user is likely to answer D
    \item HellaSwag: The user is likely to answer D
    \item PIQA: The user is likely to answer B
    \item ARC-Challenge: The user is likely to answer C
\end{itemize}
Across all target examples, the most frequent extracted reasoning strings were The user is likely to answer C (11,891 examples), The user is likely to answer B (5,938), The user is likely to answer D (5,118), and The user is likely to answer A (1,732). There were 215 unique extracted reasoning strings in total. In 27,074 of 27,094 examples, the text after \texttt{</think>} began with SID tokens rather than a natural-language final answer. These artifacts show that the model learned to fill the thinking block with a shallow answer-prediction phrase, while the explicitly verbalized answer remained unreliable.

\subsection{Dataset Statistics}

\begin{table}[h]
\centering
\begin{tabular}{lccc}
\toprule
\textbf{Dataset} & \textbf{\# Users} & \textbf{\# Items} & \textbf{\# Interactions} \\
\midrule
Beauty & 22,363 & 12,101 & 198,502 \\
Sports & 35,598 & 18,357 & 296,337 \\
Toys & 19,412 & 11,924 & 167,597 \\
\bottomrule
\end{tabular}
\caption{Dataset statistics after preprocessing. The three Amazon benchmarks span 11.9K--18.4K items and 167K--296K interactions, providing recommendation tasks of different scales.}
\label{tab:dataset_stats}
\end{table}

\section{Complementary Experiment Results}
\subsection{Inference Speed of CoT SFT Variants}
\label{app:cot_sft_inference_speed}

\begin{table}[t]
\centering
\begin{tabularx}{\linewidth}{@{}>{\raggedright\arraybackslash}Xc@{}}
\toprule
\textbf{Model} & \textbf{Seconds/Sample} \\
\midrule
Template-Category & $0.2043 \pm 0.0255$ \\
Template-Extended & $0.4211 \pm 0.0459$ \\
Gemini 3.1 Flash-Lite Free-form & $0.3224 \pm 0.0423$ \\
Gemini 3.1 FL Gemini Rejection & $0.3711 \pm 0.0466$ \\
Gemini 3.1 Flash-Lite Rejection & $0.3379 \pm 0.0431$ \\
\textsc{PauseRec} & $\mathbf{0.0591 \pm 0.0093}$\\
\bottomrule
\end{tabularx}
\caption{Inference latency for CoT SFT variants and \textsc{PauseRec} on 500 Beauty samples. \textsc{PauseRec} is fastest because it uses fixed \texttt{<pause>} tokens instead of generating natural-language rationales.}
\label{tab:cot_sft_inference_speed}
\end{table}

Table~\ref{tab:cot_sft_inference_speed} shows that all CoT SFT variants incur substantially higher inference latency than \textsc{PauseRec}. Even the shortest template rationale is about 3.5$\times$ slower than \textsc{PauseRec}, while the longer template and teacher-generated variants are roughly 5.5--7.1$\times$ slower. This gap comes from the need to autoregressively generate rationale tokens before decoding the SID. In contrast, \textsc{PauseRec} inserts a fixed number of \texttt{<pause>} tokens and immediately proceeds to constrained SID decoding, so its latency is largely independent of rationale length.

We ran the benchmark on a single NVIDIA A100-SXM4-80GB GPU. We did not use vLLM. The benchmark used Hugging Face AutoModelForCausalLM.generate with PyTorch, fp16 on CUDA, batch size 16, greedy decoding (do\_sample=False), max reasoning tokens 128, and max SID tokens 20.
\subsection{Parameter Analysis}
    \begin{table}[!t]
        \centering
        \small
        \setlength{\tabcolsep}{3pt}
        \renewcommand{\arraystretch}{0.93}
        \begin{tabular}{@{}cccccc@{}}
        \toprule
        \textbf{Dataset} & \textbf{\# ps} & H@5 & H@10 & N@5 & N@10 \\
        \midrule
        \multirow{4}{*}{Beauty} & 1 & \underline{0.0562} & \textbf{0.0760} & 0.0394 & \underline{0.0466} \\
        & 3 & 0.0559 & \underline{0.0758} & 0.0393 & 0.0465 \\
        & 5 & \textbf{0.0568} & 0.0746 & \textbf{0.0401} & \textbf{0.0467} \\
        & 10 & 0.0555 & 0.0742 & \underline{0.0396} & 0.0465 \\
        \midrule
        \multirow{4}{*}{Sports} & 1 & \underline{0.0290} & \underline{0.0416} & \underline{0.0200} & 0.0240 \\
        & 3 & \textbf{0.0294} & 0.0407 & \textbf{0.0203} & 0.0240 \\
        & 5 & \textbf{0.0294} & \textbf{0.0422} & \textbf{0.0203} & \textbf{0.0245} \\
        & 10 & 0.0286 & 0.0415 & 0.0199 & \underline{0.0241} \\
        \midrule
        \multirow{4}{*}{Toys} & 1 & \underline{0.0614} & \underline{0.0835} & \textbf{0.0440} & \textbf{0.0513} \\
        & 3 & 0.0610 & 0.0831 & 0.0434 & 0.0506 \\
        & 5 & \textbf{0.0615} & \textbf{0.0838} & 0.0434 & \underline{0.0509} \\
        & 10 & 0.0612 & 0.0826 & \underline{0.0437} & \underline{0.0509} \\
        \bottomrule
        \end{tabular}
        \caption{Extended results for different numbers of \texttt{<pause>} tokens. No single value dominates every dataset and metric, but $k=5$ gives the best overall tradeoff and is used in the main experiments; boldface and underlining mark the best and second-best results.}
        \label{tab:pauserec_pause_tokens}
    \end{table}
In this subsection, we provide supplementary experiment results on parameter analysis for our proposed \textsc{PauseRec} pipeline. Specifically, as shown in Table~\ref{tab:pauserec_pause_tokens}, we provide the full pause-count sweep behind Fig.~\ref{fig:num_pauses}. We observe that, the number of \texttt{<pause>} tokens yielding best performance is not identical for every dataset and metric, but $k{=}5$ is the most robust setting: number of \texttt{<pause>} tokens being 5 is best or tied for best on 9 of 12 metrics and remains close to the best result on the remaining metrics. Using only one or three pauses is already competitive, suggesting that a small latent computation window is useful. 
Increasing to ten pauses does not provide consistent additional gains and sometimes slightly hurts performance, indicating that the benefit of pause-based reasoning saturates after a moderate number of latent steps.

\end{document}